\documentclass[twoside]{article}

\usepackage[accepted]{aistats2025}

\usepackage{amsthm}
\usepackage{subcaption}
\usepackage[round]{natbib}
\usepackage{multirow}
\usepackage{hyperref}
\usepackage{tikz}
\usepackage{pgfplots}
\usetikzlibrary{shapes.geometric}
\usepackage{cleveref}
\usepackage{graphicx}
\usepackage{pgf}
\usepackage{enumitem}
\usepackage{tikz}
\usetikzlibrary{positioning, arrows.meta, bending}
\usepackage[utf8]{inputenc}
\usepackage[T1]{fontenc}
\usepackage{amssymb}
\usepackage{pifont}
\usepackage{dsfont}

\usepackage{color}

\usepackage[normalem]{ulem}

\usepackage{float}
\usepackage{lscape}
\usepackage{rotating}
\usepackage{longtable}

\usepackage{etoc}

\usepackage{mdframed}

\newcommand{\Ptr}[1]{P^{\text{tr}}(#1)}
\newcommand{\Pte}[1]{P^{\text{te}}(#1)}

\newtheorem{theorem}{Theorem}
\newtheorem{definition}{Definition}
\newtheorem{proposition}{Proposition}

\newtheorem{corollary}{Corollary}
\newtheorem{lemma}{Lemma}

\newtheorem{assumption}{Assumption}

\usepackage{ctable}

\def\sX{{\mathbb{X}}}
\def\sY{{\mathbb{Y}}}

\begin{document}

\runningauthor{Parjanya Prajakta Prashant, Seyedeh Baharan Khatami, Bruno Ribeiro, Babak Salimi}

\twocolumn[

\aistatstitle{Scalable Out-of-Distribution Robustness in the Presence of Unobserved Confounders}

\aistatsauthor{ Parjanya Prajakta Prashant \And Seyedeh Baharan Khatami}

\aistatsaddress{ University of California San Diego \And  University of California San Diego}

\aistatsauthor{Bruno Ribeiro \And Babak Salimi}
\aistatsaddress{Purdue University \And University of California San Diego
}

]

\begin{abstract}
    We consider the task of out-of-distribution (OOD) generalization, where the distribution shift is due to an unobserved confounder ($Z$) affecting both the covariates ($X$) and the labels ($Y$). This confounding introduces heterogeneity in the predictor, i.e., \(P(Y \mid X) = E_{P(Z \mid X)}[P(Y \mid X,Z)]\), making traditional covariate and label shift assumptions unsuitable. OOD generalization differs from traditional domain adaptation in that it does not assume access to the covariate distribution ($X^\text{te}$) of the test samples during training. These conditions create a challenging scenario for OOD robustness: (a) $Z^\text{tr}$ is an unobserved confounder during training, (b) $\Pte{Z} \neq \Ptr{Z}$, (c) $X^\text{te}$ is unavailable during training, and (d) the predictive distribution depends on $\Pte{Z}$. While prior work has developed complex predictors requiring multiple additional variables for identifiability of the latent distribution, we explore a set of identifiability assumptions that yield a surprisingly simple predictor using only a single additional variable. Our approach demonstrates superior empirical performance on several benchmark tasks.
\end{abstract}

\section{INTRODUCTION}
\label{sec:introduction}
In this paper, we explore a class of out-of-distribution (OOD) tasks in which a latent variable (\( Z \)) confounds the relationship between covariates (\( X \)) and the label (\( Y \)). This confounding creates heterogeneity in the data, causing the optimal predictor \( \hat{Y} = f_Z(X) \) to depend on \( Z \). Consequently, when the distribution of \( Z \) undergoes a shift, the joint distribution \( P(X, Y) \) shifts accordingly. Such shifts often lead to poor generalization for models trained using empirical risk minimization~\citep{quinonero2022dataset}.
Addressing OOD tasks in the presence of latent confounders is particularly challenging because \( Z \) is unobserved during both training and testing, the latent distribution changes (\( \Ptr{Z} \neq \Pte{Z} \)), and the posterior predictive distribution depends on \( \Pte{Z} \) (i.e., \( \hat{Y} = \mathbb{E}_{Z \sim \Pte{Z \mid X}}[f_Z(X)] \)).

Traditional approaches primarily focus on distribution shifts in \( X \) or \( Y \), assuming \( Z \) affects only one of them~\citep{Shimodaira2000, Lipton2018, IRM, DROwithlabel, HetRiskMin}. Recent studies investigate scenarios where shifts in the latent confounder (\( Z \)) affect both \(X\) and \(Y\)~\citep{Alabdulmohsin2023, ProxyDA}. These methods typically require auxiliary information through proxy variables (noisy measurements of \( Z \)) and concept variables (concept bottleneck mediating the dependence between \( X \) and \( Y \)). However, in many practical cases, multiple auxiliary variables might not exist or be observable. Furthermore, these approaches often rely on complex and hard-to-train techniques, such as generative models or kernel methods~\citep{lu2014scale, lucas2019understanding, paulus2020rao}.

Motivated by these challenges, we propose a simpler and scalable approach which requires only one additional variable. Inspired by Box's principle that "all models are wrong, but some are useful," we establish identifiability of the latent distribution—and consequently the predictor—based on a set of practical assumptions. Unlike prior work, we do not require a concept variable and demonstrate that identifiability is achievable with only a proxy variable or multiple data sources. Based on our identifiability theory, we introduce a novel regularizer, which is crucial for ensuring the convergence of our model, as shown in \Cref{sec:results}.

\textbf{Contributions.} Our work addresses OOD generalization when there is a shift in the latent confounder distribution. We theoretically establish the identifiability of latent distribution using only a single additional variable (either a proxy or data from multiple sources), relaxing requirements from prior work that require multiple auxiliary variables. Based on this theoretical foundation, we develop a two-stage approach: First, we learn the latent confounder distribution (\(\Ptr{Z \mid X}\)) using an encoder-decoder architecture. The encoder models the posterior distribution of the confounder (\(\Ptr{Z \mid X}\)) and the decoder reconstructs the proxy variable distribution (\(P(S \mid Z)\)). Second, we train a mixture-of-experts model where each expert specializes in a specific confounder assignment. At inference time, we address distribution shift by estimating \(\Pte{Z}\) and appropriately reweighting \(\Ptr{Z \mid X}\) to obtain \(\Pte{Z \mid X}\). We outperform existing baselines on both synthetic and real data.

\section{RELATED WORK}
\label{sec:related_work}

\paragraph{\bf Invariant OOD methods} Invariant OOD methods focus on learning representations that remain stable across training and test environments. Techniques such as IRM~\citep{IRM}, GroupDRO~\citep{DROwithlabel}, HRM~\citep{HetRiskMin}, and V-REx~\citep{krueger2021out} leverage environment labels to train an invariant predictor. \citet{wang2022out} learn an invariant predictor using knowledge of causal features. Other invariant approaches do not require environment labels~\citep{nam2020learning, EIIL, liu2021just, nam2022spread}. \citet{relaxedIRM} relax the strict invariance assumption but still focus on a single predictor for all environments. Additionally, \citet{kaur2022modeling} extend invariant prediction to multiple shifts. Nevertheless, empirical studies have demonstrated that such methods may underperform compared to empirical risk minimization~\citep{rosenfeld2020risks, domainbed}, and theoretically, the invariant predictor assumption does not hold in many scenarios~\citep{Schrouff2022, lin2022zin}.

\paragraph{\bf Proxy Methods} Proxy-based approaches have recently been employed to address sub-population shifts~\citep{Alabdulmohsin2023, ProxyDA}, which assume access to proxy and concept variables to make accurate predictions under environmental shifts. \citet{ProxyDA} eliminate the need for concept variables if multiple labeled domains with proxy variables are accessible. However, these methods assume that proxy variables are observable in the test data and that test data are available during training. Proxy methods have also been explored in causal inference~\citep{Miao, tchetgen2020introduction, xu2023kernel}, where \citet{Miao} and \citet{tchetgen2020introduction} establish identifiability when multiple proxies for the environment are observed. \citet{xu2023kernel} assume the label (\(Y\)) is a deterministic function of \(X\) and \(Z\). Our approach differs by relaxing several assumptions: (1) we require only one additional variable - a proxy variable or access to multiple training domains; (2) we do not require access to the test distribution during training; (3) proxy variables are not required in the test set; (4) we allow both causal and anti-causal relationships between features and the label; and (5) we offer a simpler algorithm that avoids generative modeling or kernel methods.

\paragraph{\bf Domain Adaptation} Domain adaptation methods typically assume access to the test distribution during training. Many of these methods seek to learn an invariant feature space by assuming a stable predictor~\citep{deepcoral, dann}. \citet{garg2023rlsbench} consider relaxed-label shift, where along with label shift, $P(X \mid Y)$ changes in a constrained manner. Test-time adaptation methods have recently been developed to modify models during inference to match the test distribution~\citep{arm, testtimeentropy, testtimeself, testtimetricks}. For instance, \citet{arm} adopt a meta-learning strategy, while \citet{testtimeentropy} and \citet{testtimeself} retrain the model to minimize entropy and use self-supervision at test time, respectively. Recently, test-time adaptation methods addressing spurious correlations have been proposed~\citep{sun2023beyond, eastwood2024spuriosity}. However, these methods do not address confounder shifts~\citep{Alabdulmohsin2023}.

\paragraph{\bf Causality in OOD and Domain Adaptation} Causal reasoning has been applied to OOD generalization and domain adaptation problems ~\citep{causalanticausal,Zhang2015, scholkopf2021toward, causaloodinv2, latentcausal, InvcausalOOD, mahajan2021domain}. Causal models help elucidate distribution shifts and identify invariant relationships ~\citep{causalanticausal, causaloodinv2}. \citet{latentcausal} assume that latent space can be disentangled into spurious and causal features and learn the causal features from multiple labeled sources. \citet{huang2022harnessing} similarly disentangle content and style features from image data leveraging multiple domains. \citet{CausalBalancing} re-weight data to remove confounding effects between labels and covariates. They assume \(Y \rightarrow X\) and require several training domains for identifiability. \citet{gao2023robust} eliminate observed confounder effects in graph representation learning. However, general shifts in the latent confounder distribution are not considered in these works.

\section{NOTATION AND PROBLEM FORMULATION}
\label{sec:formulation}

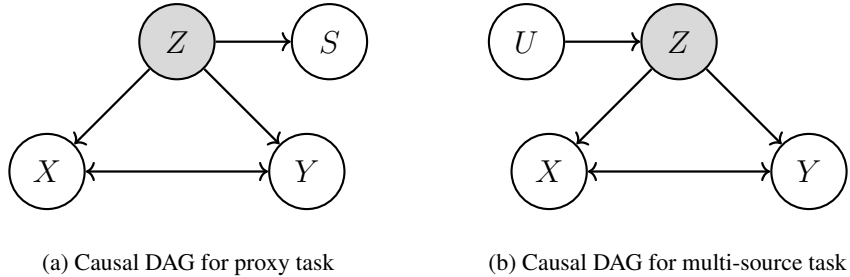
\begin{figure*}[t]
\centering
\begin{subfigure}[t]{0.45\textwidth}

\centering
\begin{tikzpicture}[node distance=2cm and 2cm, scale=0.1]
  \node[draw, ellipse, fill=gray!30, minimum size=1cm, thick] (Z) {\large $Z$};
  \node[align=left, right=0.1cm of Z] (Z_desc) {};

  \node[draw, ellipse, fill=white, right= 1cm of Z, minimum size=1cm, thick] (S) {\large $S$};
  \node[align=left, below=0.1cm of S] (S_desc){};

  \node[draw, ellipse, fill=white, below left=1cm and 1cm of Z, minimum size=1cm, thick] (X) {\large $X$};
  \node[align=left, below=0.1cm of X] (X_desc){};

  \node[draw, ellipse, fill=white, below right=1cm and 1cm of Z, minimum size=1cm, thick] (Y) {\large $Y$};
  \node[align=left, below=0.1cm of Y] (Y_desc) {};

  \draw[->, line width=0.3mm] (Z) -- (S);
  \draw[->, line width=0.3mm] (Z) -- (X);
  \draw[->, line width=0.3mm] (Z) -- (Y);
  \draw[<->, line width=0.3mm] (X) -- (Y);
\end{tikzpicture}
\caption{Causal DAG for proxy task}
\label{fig:DAG1a}
\end{subfigure}
\begin{subfigure}[t]{0.45\textwidth}
\centering
\begin{tikzpicture}[node distance=2cm and 2cm, scale=0.1]
  \node[draw, ellipse, fill=gray!30, minimum size=1cm, thick] (Z) {\large $Z$};
  \node[align=left, right=0.1cm of Z] (Z_desc) {};

  \node[draw, ellipse, fill=white, left= 1cm of Z, minimum size=1cm, thick] (U) {\large $U$};
  \node[align=left, below=0.1cm of U] (U_desc){};

  \node[draw, ellipse, fill=white, below left=1cm and 1cm of Z, minimum size=1cm, thick] (X) {\large $X$};
  \node[align=left, below=0.1cm of X] (X_desc){};

  \node[draw, ellipse, fill=white, below right=1cm and 1cm of Z, minimum size=1cm, thick] (Y) {\large $Y$};
  \node[align=left, below=0.1cm of Y] (Y_desc) {};

  \draw[->, line width=0.3mm] (U) -- (Z);
  \draw[->, line width=0.3mm] (Z) -- (X);
  \draw[->, line width=0.3mm] (Z) -- (Y);
  \draw[<->, line width=0.3mm] (X) -- (Y);
\end{tikzpicture}
\caption{Causal DAG for multi-source task}
\label{fig:DAG1b}
\end{subfigure}
\caption{Causal diagrams. Shaded circles denote unobserved variables and solid circles denote observed variables. $X$ is the covariate, $Y$ is the label, $Z$ is the confounder, $S$ is the proxy for the confounder (in the proxy task), and $U$ is the dataset id for the multi-source task.}
\label{fig:DAGs}
\end{figure*}

In this section, we introduce our tasks and formally define our notation.

\paragraph{Notation} Let \( X \in \sX \subseteq \mathbb{R}^{n_x} \) be the input variable (e.g., an image, a graph, a set), \( Y \in \sY \) be the target variable, and \( Z \in [n_z] \) with \( n_z \geq 2 \) (where \([n_z] := \{1, \ldots, n_z\}\)) denote the latent confounder that influences both \(X\) and \(Y\). Examples of such confounders include socioeconomic factors, demographic characteristics, or institutional policies. Let \( P^{tr} \) and \( P^{te} \) denote train and test distributions respectively. We use \( P \) for distributions that are invariant across train and test.  Under these settings, we consider two tasks, that we give arbitrary names for clarity.

\paragraph{Proxy OOD Task} This task is described by the causal DAG shown in Figure \ref{fig:DAG1a}. We have an additional variable (\(S\)) which is a proxy for the latent confounder (\(Z\))~\citep{pearlmeasurementbias, kuroki2014measurement}. To illustrate this setup, consider the example from \citet{kuroki2014measurement} where participation in a social services program (\(X\)) affects children's cognitive ability (\(Y\)), with socioeconomic status (\(Z\)) acting as a latent confounder. In this scenario, father's occupation and family income, collectively denoted as (\(S\)), serve as proxy variables for the latent confounder (\(Z\)). The goal is to train a model that generalizes to test data where the distribution of the latent confounder has shifted from its training distribution.

\begin{definition}[\textbf{Proxy OOD}] \label{def:OODsetting1}
    Assume the data generation process for the input variable (\(X\)), the target label (\(Y\)), the latent confounder variable (\(Z\)), and proxy (\(S\)) follows the causal DAG in Figure~\ref{fig:DAG1a}. Let the training data \(D^{\text{tr}} = \{(x^i, y^i, s^i)\}_{i=1}^{N^{\text{tr}}} \sim P^{\text{tr}}(X, Y, S)\) be collected with respect to the training confounder distribution \(P^{\text{tr}}(Z)\), i.e., \(P^{\text{tr}}(X, Y, S) = \sum_{Z \in [n_z]} P(X, Y, S \mid Z) P^{\text{tr}}(Z)\). In the OOD test data, the confounder distribution shifts, i.e., \(P^{\text{tr}}(Z) \neq P^{\text{te}}(Z)\), while the causal relationships, i.e., \(P(S \mid Z), P(X \mid Z), P(Y \mid Z), P(X , Y \mid Z) \) remain invariant. Our goal is to learn a predictor that can accurately predict the labels in a test dataset \(D^{\text{te}} = \{x^i\}_{i=1}^{N^{\text{te}}} \sim P^{\text{te}}(X)\) with a new test confounder distribution \(P^{\text{te}}(Z)\).
\end{definition}

    We further assume the following conditions:
    \begin{enumerate}
    \item \textbf{Discrete Proxy:} $S \in [n_s]$. If $S$ is continuous, we can discretize it.
    \item \textbf{Confounder Support:} \(\text{supp}(\Pte{Z}) \subseteq \text{supp}(\Ptr{Z})\).
    \end{enumerate}

\paragraph{Multi-source OOD Task} This task does not assume any observed proxies but multiple training datasets from different sources or domains with varying distributions of the confounder as shown by the causal DAG in Figure \ref{fig:DAG1b}. We require only one domain to be labeled. The task is to train a model that generalizes to unseen domains with the domain causing a shift in the distribution of the latent confounder (\(Z\)). For example, consider data from various health centers to predict healthcare outcomes such as disease prevalence (\(Y\)) based on attributes like age, gender and diet (\(X\)). Socio-economic factors (\(Z\)) confound the relationship between \(X\) and \(Y\). Suppose only the data from one health center is labeled due to resource constraints. The goal is to use labeled data from one health center and unlabeled data from other centers to train a model that can generalize to predict outcomes in an unseen health center.

\begin{definition}[\textbf{Multi-Source OOD}] \label{def:OODsetting2}
    Assume the data generation process for the input variable \(X\), the target label \(Y\), the latent confounder \(Z\), and dataset identifier \(U\) follows the causal DAG in Figure~\ref{fig:DAG1b}. Let \(P_k(X, Y)\) denote the distribution for the \(k^{\text{th}}\) dataset, i.e., \(P_k(X, Y) = \sum_{Z \in [n_z]} P(X, Y \mid Z) P(Z \mid U = k)\). We assume we have access to \(P_1(X, Y)\) and \(P_k(X)\) for \(2 \leq k \leq n_u\). In the OOD test data, we assume we see a new distribution, i.e., \(P^{\text{te}} = P_{n_u+1}(X)\), while the causal relationships, i.e., \(P(X \mid Z), P(Y \mid Z), P(X , Y \mid Z) \) remain invariant. The OOD extrapolation goal is to learn a predictor that can accurately predict the missing target labels in a test dataset \(D^{\text{te}} = \{x^i\}_{i=1}^{N^{\text{te}}} \sim P^{\text{te}}(X)\).
\end{definition}

    We further assume the following conditions:
\begin{enumerate}
    \item \textbf{Confounder support:} \(\text{supp}\big(P(Z \mid U = n_u + 1)\big) \subseteq \text{supp}\big(P(Z \mid U = 1)\big)\).
    \item \textbf{Distinct domains:} \(P(Z \mid U = k) \neq P(Z \mid U = l)\).
\end{enumerate}

\section{THEORY}\label{sec:Identifiability}
In this section, we establish identifiability of the latent distribution \(P(Z \mid X)\) given observed distribution \(P(X,S)\). We first present sufficient conditions for identifiability and discuss them in the context of existing work. We then state our main result that the latent distribution is approximately identifiable with an error bound that depends on \(\max_x P(Z \mid X=x)\). Finally, we demonstrate that this bound asymptotically approaches zero as the dimensionality of X increases, making the latent distribution fully identifiable for high-dimensional data. While we discuss identifiability with respect to the Proxy OOD task, the same results hold for the Multi-Source OOD task.

\subsection{Assumptions}
Latent confounder distributions with a single proxy variable cannot be identified in general~\citep{pearlmeasurementbias}. In order to prove identifiability, we consider the following assumptions:

\begin{assumption}[Structural Constraints] \label{asp:structural_constraints}
The Graphs in Figure~\ref{fig:DAGs} are Markov.~\citep{spirtes2000causation}.
\end{assumption}

This assumption implies the following conditional independence property: \(S \perp X \mid Z\).

\begin{assumption}[Proxy Distribution Rank] \label{asp:proxy_rank}
Let \( M \) denote the matrix of conditional probabilities \( P(S \mid Z) \), where each entry \( M_{zs} \) is given by \( P(S = s \mid Z = z) \). Then, \( \text{rank}(M) = n_{z} \).
\end{assumption}

This assumption ensures the identifiability of the latent confounder \( Z \). Without this condition, degenerate solutions where each data point is assigned to its own distinct latent class (i.e., \( \forall i, Z_i = i )\) would be possible. Moreover, this implies \( n_s \geq n_z \) which guarantees that the proxy variable contains enough information to recover the latent variable.

\begin{assumption}[Weak Overlap] \label{asp:weak_overlap}
There exists \(\eta \in \left( \frac{1}{2}, 1 \right]\) such that, for each \(i \in \{1, 2, \dots, n_z\}\),

\[
\max_{x} P(Z = i \mid X = x) \geq \eta.
\]

\end{assumption}

Intuitively, this assumption requires that for each latent class \(Z = i\), there exists a region in the feature space where that class has high posterior probability. Specifically, for each class, there is some observation \(X = x\) where \(P(Z = i \mid X = x) \geq \eta > \frac{1}{2}\), meaning the class has the highest probability in that region. This condition prevents excessive overlap between latent variables and ensures they remain distinguishable based on the observed data.

\paragraph{Discussion}
Our assumptions are fairly general. Assumption~\ref{asp:structural_constraints} and~\ref{asp:proxy_rank} are common in the proximal causal inference literature~\citep{spirtes2000causation, Miao, Alabdulmohsin2023, ProxyDA}. Assumption~\ref{asp:weak_overlap} may seem restrictive, however in the subsequent subsection, we will show that it is reasonable for high-dimensional data. In contrast to prior works such as \citet{ProxyDA} and \citet{Miao}, we do not make the informative variable assumption (requiring that \( P(Z \mid X) \) does not take a finite set of values) or the assumption of additional variables like the concept variables.  We provide further discussion in Appendix~\ref{appendix:assumption_discussion}.

\subsection{Identifiability}
We show that the latent distribution is approximately identifiable, with the identifiability error being a function of $\eta$ from Assumption~\ref{asp:weak_overlap}. As $\eta$ increases, indicating less overlap, the error decreases. For large values of $\eta$ (i.e., $\eta \to 1$), the error approaches zero and the latent distribution becomes increasingly identifiable.

\begin{theorem}[Approximate Identifiability of Latent Variables]
\label{theorem:identifiability}

Let $P(X, S, Z)$ be the true distribution and $Q(X, S, Z)$ be any other distribution over the observed variables $X$, proxy variables $S$, and latent variables $Z$. Suppose that the marginal distributions over $X$ and $S$ are identical, i.e., $P(X, S) = Q(X, S)$, and that Assumptions~\ref{asp:structural_constraints}-~\ref{asp:weak_overlap} are satisfied for both $P$ and $Q$. Then, there exists a permutation $\pi$ of $\{1, 2, \dots, n_z\}$ such that for each $i \in \{1, 2, \dots, n_z\}$,  conditional distributions $P(Z \mid X = x)$ and $Q(Z \mid X = x)$ satisfy the following bound:

\begin{align}
\sup_{x \in \text{supp}(X)} \left| P(Z = i \mid X = x) - Q(Z = \pi(i) \mid X = x) \right|
& \nonumber \\  \leq O\left( \frac{1 - \eta}{2\eta - 1} \right) \label{eq:bound}
\end{align}

where $\eta \in \left( \frac{1}{2}, 1 \right]$ is as defined in Assumption \ref{asp:weak_overlap}.

\end{theorem}

\paragraph{Proof Sketch}
Our proof proceeds in three steps.
\begin{enumerate}[leftmargin=*]
    \item \textbf{Uniqueness of \(n_z\)}: Given the observed distribution \(P(X,S)\), observe that since \(S\) is discrete, \(P(S\mid X=x)\) is a vector in \(\mathbb{R}^{n_s}\). The set of these vectors \(\{P(S \mid X=x)\}_{x \in \text{supp}(X)}\) spans a subspace in \(\mathbb{R}^{n_s}\). We show the dimension of this subspace equals \(n_z\). This establishes the unique dimensionality of the latent space.

    \item \textbf{Linear relationship of conditional distributions}: By the conditional independence \(S \perp X \mid Z\), we have \(
    P(S \mid X=x) = \sum_{z \in [n_z]} P(S \mid Z=z)P(Z=z \mid X=x),
    \)
    with an analogous expression for \(Q\). Given \(P(S \mid X) = Q(S \mid X)\), it follows that a linear transformation \(A\) relates the distributions as
    \[
    P(Z \mid X=x) = A\,Q(Z \mid X=x).
    \]

    \item \textbf{Approximation via permutation matrix}: Under Assumption~\ref{asp:weak_overlap}, we show that matrix \(A\) can be approximated by a permutation matrix up to an error bound of order
    \[
    O\left(\frac{1 - \eta}{2\eta - 1}\right).
    \]
    This leads to a bound on the difference between \(P(Z \mid X = x)\) and \(Q(Z \mid X = x)\), yielding the approximate identifiability result.
\end{enumerate}

The formal proof is provided in Appendix~\ref{appendix:approximate_identifiability_proof}.

The bound established in Step 3 ensures that \( P(Z \mid X) \) and \( Q(Z \mid X) \) are approximately related by a permutation matrix. In the limiting case where \(\eta = 1\), the approximation becomes exact, leading to full identifiability. This is formalized in the following corollary.
\begin{corollary}[Full Identifiability]
\label{corollary:full_identifiability}

If $\eta = 1$, then the latent variables are fully identifiable. That is, there exists a permutation $\pi$ such that, for all $x \in \text{supp}(X)$ and for each $i \in \{1, 2, \dots, n_z\}$,

\[
P(Z = i \mid X = x) = Q(Z = \pi(i) \mid X = x).
\]

\end{corollary}

Next, we argue that as the dimensionality of $X$ increases, $\eta \to 1$, making the latent distribution more identifiable. Intuitively, higher-dimensional covariates provide more information, making it easier to distinguish between different latent classes, naturally driving $\eta$ closer to 1 and reducing overlap. This contrasts with the strict overlap assumption often used in causal inference literature, where it is typically assumed that $1 - \eta \leq P(Z \mid X) \leq \eta$ for some $\eta \in (0,1)$~\citep{van2011targeted}.
It has been shown that as dimensionality increases, satisfying strict overlap is increasingly difficult~\citep{d2021overlap}. Inspired by this, we show that as the dimensionality of $X$ increases and the features in $X$ are sufficiently informative, $\eta=\max_{x} P(Z \mid X = x)$ approaches 1. This result builds on Proposition 2 of \citet{d2021overlap}.

\begin{proposition} \label{prop:eta}
    Let \( Z \in [n_z] \) and \( X \in \mathbb{R}^{dim(X)} \). Further, assume that each feature \( X_{i} \) is sufficiently discriminative of the latent values of \( Z \), conditioned on all previous features \( X_{1:i-1} \). Specifically, for all feature indices \( i \in \{1,...,dim(X)\} \) and all latent values \( j \in [n_z] \), there exists \(\epsilon > 0\) such that, for all \( k \in [n_z] \) where \( k \neq j \),
    \begin{small}
    \(\mathbb{E}_{P(X|Z=j)}\text{KL}\left(P(X_i|X_{1:i-1},Z=j) \parallel P(X_i|X_{1:i-1}, Z=k)\right)
    \)
    \end{small} is greater than \(\epsilon\).
    Then, as \( dim(X) \rightarrow \infty \), there exists no \( \eta \in (0,1) \) such that
    \[    \max_{x} P(Z = j \mid X = x) \leq \eta \quad \forall j\]
\end{proposition}
The proof of this proposition is provided in Appendix~\ref{appendix:proposition_proof}.

\section{METHODOLOGY}
\label{sec:methods}
In this section, we formally introduce our proposed framework. Section~\ref{subsection:training} outlines the training phase, where we first estimate the latent conditional distribution \(\Ptr{Z|X}\) and then learn the predictive distribution \(\Ptr{Y|X}\). Section~\ref{subsection:testing} describes the inference phase, where we estimate the shifted latent marginal \(\Pte{Z}\) and use it to derive \(\Pte{Y|X}\) for prediction.

Our methodology applies uniformly to both Proxy OOD and Multi-Source OOD tasks. For clarity, we detail the process in the Proxy OOD setting, as the same steps extend to the Multi-Source OOD case.
\begin{figure*}[t]
    \vspace{-10pt}
    \centering
    \includegraphics[scale=0.35]{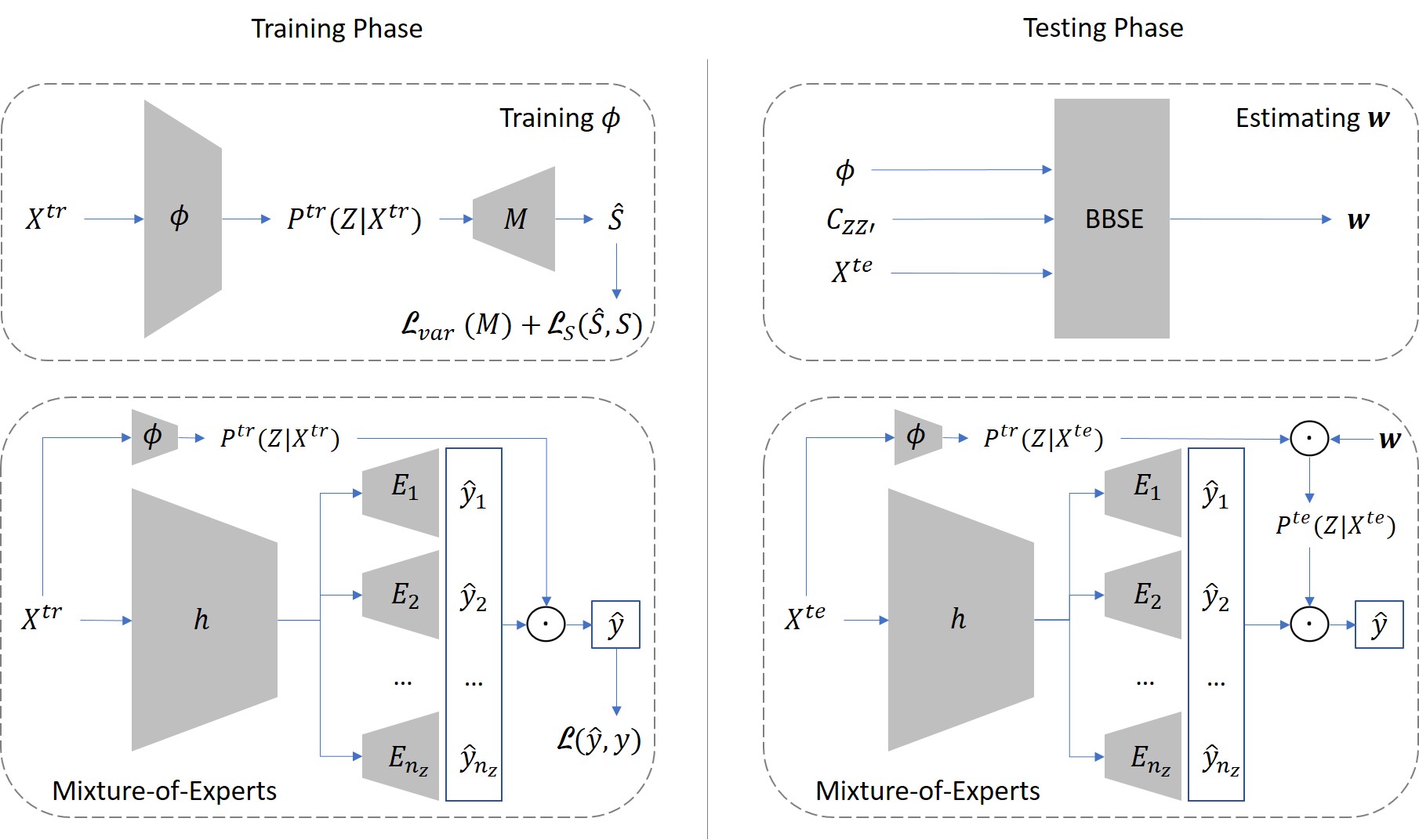}
    \caption{This figure illustrates the methodology of our proposed model, which is structured into two primary phases: the Training Phase, where we first estimate the latent distribution and subsequently train a mixture-of-experts model for prediction; and the Testing Phase, where the model estimates \(\frac{\Pte{Z}}{\Ptr{Z}}\) (\(\mathbf{w}\)) and reweights the gating function to predict the output for test input \(X^{te}\).}
    \label{im:framework}
    \vspace{-15pt}
\end{figure*}

\subsection{Training}
\label{subsection:training}
The training procedure consists of two stages. First, we train an encoder-decoder model to infer the latent confounder distribution \( \Ptr{Z \mid X} \). This step factorizes the observed conditional distribution \( \Ptr{S \mid X} \) into \( \Ptr{Z \mid X} \) and \( P(S \mid Z) \), using appropriate regularization to ensure identifiability. Next, we train a mixture-of-experts (MoE) model to estimate \( \Ptr{Y \mid X} \). The inferred distribution \( \Ptr{Z \mid X} \) weights the experts, so each expert specializes in distinct latent classes and effectively captures data heterogeneity.

\subparagraph{Estimating Latent Distribution:}
Encoder-decoder models are often used to obtain the latent distributions~\citep{kingma2013auto}. We design an encoder-decoder model based on the independence constraints for our task (Assumption~\ref{asp:structural_constraints}).

Assumption~\ref{asp:structural_constraints} gives us \( S \perp X \mid Z \). Therefore, \( \Ptr{S \mid X = x} \) can be factorized as follows:
\begin{small}
\begin{equation}
    \Ptr{S \mid X = x} = \sum_{z \in [n_z]} P(S \mid Z = z) \Ptr{Z = z \mid X = x}, \label{eq:fac}
\end{equation}
\end{small}
\\
This factorization decomposes the distribution \( \Ptr{S \mid X} \) into two components: \( \Ptr{Z \mid X} \) and \( P(S \mid Z) \). The latent variable \( Z \) serves as an intermediate bottleneck that captures the information of \(X\) relevant to the prediction of \(S\). We parameterize \( \Ptr{Z \mid X} \) using a neural network encoder \( \phi: \mathbb{R}^{dim(X)} \rightarrow [0,1]^{n_z} \). The encoder transforms \( X \) into a distribution over \( Z \) using a softmax function in the final layer. The conditional distribution \( P(S \mid Z) \), which is a discrete distribution, is then modeled by a matrix \( M \in [0,1]^{n_z \times n_s} \) where each entry \(M_{zs}\) represents \(P(S = s \mid Z = z)\). This acts as a decoder mapping \( Z \) to \( S \).

The parameters \(\phi\) and \(M\) are learned by solving the following optimization problem:
\begin{equation}
\phi^*, M^* = \arg\min_{\phi, M} \mathbb{E}_{(X, S) \sim \mathcal{D}^{\text{tr}}} \left[ \mathcal{L}_{S}(M^\intercal \phi(X), S) \right] \label{eq:recloss}
\end{equation}

where \( M^\intercal \phi(x) = \sum_{z} \phi_z(x) M_{z,:} \), \( \phi_z(x) \) denotes the \( z^{\text{th}} \) component of \( \phi(x) \), and \( M_{z,:} \) represents the \( z^{\text{th}} \) row of \( M \) and \(\mathcal{L}_{S}\) is the loss function.

The factorization in \Cref{eq:fac} is generally not unique, since multiple pairs $(M, \phi)$ can minimize the loss defined in \Cref{eq:recloss}. However, under \Cref{asp:weak_overlap} with \(\eta = 1\), unique identification becomes possible. Recall that this assumption requires the existence of at least one \(x \in \text{supp}(X)\) for each latent class \(j\) satisfying \(P(Z = j \mid X = x) = 1\). Directly verifying this assumption is computationally infeasible for large datasets due to the intractable number of potential \(X\) values. To address this issue, we introduce a regularization term on the matrix \(M\). This term computes the maximum variance among the rows in \(M\). The following theorem shows that this enforces Assumption~\ref{asp:weak_overlap} thereby ensuring unique identification of the latent distribution.

\begin{theorem} \label{theorem:regularization}
Consider the set \( \mathcal{M} \), which denotes the set of all matrices \( M \) associated with some distribution \( P(S \mid Z) \) for which there exists a \(\phi(\cdot)\) associated with density \( \Ptr{Z \mid X} \) such that \( M \) and \(\phi(\cdot)\) are minimizers of the reconstruction loss in \Cref{eq:recloss}. Let \( M^* \in \mathcal{M} \) be the matrix associated with a \( Z \) that satisfies Assumption~\ref{asp:weak_overlap} for \(\eta = 1\). Then it holds that
\[
 \mathcal{L}_{\text{var}}(M^{*}) \leq \mathcal{L}_{\text{var}}(M), \quad \forall M \in \mathcal{M},
\]
where
\begin{equation}
    \label{eq:Lvar}
\mathcal{L}_{\text{var}}(M) = \max_{z \in [n_z]} \left( \frac{1}{n_s} \sum_{s=1}^{n_s} \left(M_{zs} - \frac{1}{n_s}\right)^2 \right)
\end{equation}
\end{theorem}

\Cref{theorem:regularization} states that under \Cref{asp:weak_overlap} for \(\eta = 1\), the true \(M\) has the minimum highest row variance. The formal proof is provided in Appendix~\ref{appendix:regularization_proof}.

Therefore, our final optimization objective is
\begin{equation}
\label{eq:final}
   \phi^*, M^* = \arg\min_{\phi, M} \mathcal{L}_{S}(M^T \phi(X), S) + \lambda \cdot \mathcal{L}_{\text{var}}(M)
\end{equation}
with $\mathcal{L}_{S}$ as in \Cref{eq:recloss} and $\mathcal{L}_{\text{var}}$ as in \Cref{eq:Lvar}.

\paragraph{Mixture-of-Experts model.}
The predictive distribution \(\Ptr{Y \mid X}\) can be written as:
\[
\Ptr{Y \mid X} = \sum_{i=1}^{n_z} \Ptr{Z = i \mid X} P(Y \mid Z = i, X)
\]
This motivates us to use a mixture-of-experts (MoE) model to model the posterior predictive distributions with estimated latent distribution \( \Ptr{Z \mid X} \) functioning as the mixture weights.

MoE models consist of multiple expert neural networks \( E_i(x) \) whose outputs are weighted and combined by a gating function based on input features~\citep{jacobs1991adaptive}. In this paper, the mixture distribution \( P^{\text{tr}}(Z \mid X) \) serves as the gating mechanism, assigning weights to each expert's output. Consequently, each expert \( E_i(\cdot) \) estimates the conditional distribution \( P(Y \mid Z = i, X) \), making each expert specialized for a specific latent class. However, training separate experts for each value of \( Z \) may lead to overfitting, particularly when certain latent classes have limited samples. To mitigate this, we share all layers across the experts except for the final layer. The shared layers learn a common representation of the input features \( X \), following the intuition in \citet{kirichenko2022last} that hidden representations up to the last layer are sufficient to adapt to different latent classes.

The prediction for a particular input \( x \) by our MoE model is given by:
\[
\text{MoE}(x) = \sum_{i=1}^{n_z} \Ptr{Z = i \mid X = x} E_i(h(x)),
\]
Here, \( h(x) \) denotes the shared lower layers that learn a common representation of the input features, while \( E_i(h(x)) \) represents the output of the \( i \)-th expert network. The mixture weights \( P(Z = i \mid X = x) \) determine the contribution of each expert's output to the final prediction. Figure~\ref{im:framework} shows an illustration of the architecture.

Training this mixture model is key to our approach as it captures heterogeneity more effectively than treating \( Z \) merely as an additional input feature. This advantage arises from explicitly factoring the predictive model into the gating function \( \Ptr{Z \mid X} \) and the expert distributions \(P(Y | Z, X)\). In both of our tasks, distributional shifts involving the latent confounder affect only the gating function \( \Ptr{Z \mid X} \), while the conditional distribution \(P(Y | Z, X)\) remains invariant across training and test settings. Consequently, adapting the model to shifts in the latent confounder can be accomplished by appropriately re-weighting the gating function without modifying the expert distributions.

\subsection{Testing}
\label{subsection:testing}
During inference, there is a shift in \( P(Z) \) from train to test, meaning that \( \Ptr{Z \mid X} \neq \Pte{Z \mid X} \). This shift affects the gating function of our MoE model. Since \( Z \) causes \( X \), \( \Ptr{X \mid Z} = \Pte{X \mid Z} \). This scenario is analogous to a label shift task, where \( P(X \mid Y) \) stays invariant but there is a shift in \( P(Y) \) from train to test~\citep{Alabdulmohsin2023}. Consequently, techniques developed for label shift can be applied to adjust for the shift in \(Z\). In this case, \(\Pte{Z \mid X}\) can be obtained by re-weighting \(\Ptr{Z \mid X}\) as follows:

\begin{equation}
\begin{aligned}
\Pte{Z=z \mid X} & = \frac{\Ptr{Z=z \mid X}\frac{\Pte{Z=z}}{\Ptr{Z=z}}}{\sum_{z'} \Ptr{Z=z' \mid X}\frac{\Pte{Z=z'}}{\Ptr{Z=z'}}}
\end{aligned}
\end{equation}
This re-weighting, however, relies on the density ratio \(\frac{\Pte{Z}}{\Ptr{Z}}\) which we denote by (\(\mathbf{w}\)). We use Black Box Shift Estimation (BBSE), a method-of-moments approach introduced by \citet{Lipton2018} to estimate this ratio. The intuition behind their approach is that, since \( Z \) causes \( X \), \( P(f(X) \mid Z) \) remains invariant from training to testing for any function \( f(X) \). By observing the changes in the marginal distribution of \( f(X) \), we can accurately estimate the shift in \( P(Z) \).

The motivation for learning \(\Pte{Z \mid X}\) through re-weighting \(\Ptr{Z \mid X}\) instead of estimating \( Z \) with our encoder-decoder framework from test data is twofold:
1) We aim to capture and adapt to changes in the underlying distribution of \(Z\) between training and test phases.
2) We do not assume access to proxy variables or multi-source data during the test phase.

\section{EXPERIMENTS}\label{sec:results}
In this section, we demonstrate the effectiveness of our approach on both synthetic and real data.

\paragraph{Baselines}
We compare our method against several baselines, focusing on out-of-distribution (OOD) and domain adaptation techniques. For OOD methods, we include Invariant Risk Minimization (IRM)~\citep{IRM}, Variance Risk Extrapolation (VREx)~\citep{krueger2021out}, and GroupDRO~\citep{DROwithlabel}, which have access to the latent confounder during training. For domain adaptation, we consider Domain-Adversarial Neural Networks (DANN)~\citep{dann} and DeepCORAL~\citep{deepcoral}. Additionally, we include Proxy Domain Adaptation (ProxyDA)~\citep{ProxyDA} as a proxy-based domain adaptation method. It is important to note that some methods have access to test data during training: ProxyDA, DeepCORAL, and DANN use test data for domain adaptation purposes. Furthermore, GroupDRO, IRM, and VREx utilize ground-truth \(Z\) labels during training, providing them with additional information about the dataset structure. ProxyDA requires both proxies and source/domain labels, but our approach only uses one. To ensure a fair comparison for tasks, we provide ProxyDA with random noise for the variable which is not available. Further details on the baselines and datasets can be found in Appendix~\ref{appendix:implementation_details}.

\paragraph{Synthetic Data}
For the proxy OOD task, we define \( Z \in \{0, 1, 2\} \), with \( S \) (proxy) \( \in \{0, 1, 2\} \), and \( X \in \mathbb{R}^3 \), while \( Y \in \{0, 1\} \). The distribution of \( Z \) shifts from \( P(Z) = [0.15, 0.35, 0.5] \) during training to \( P(Z) = [0.7, 0.2, 0.1] \) during testing. For the multi-source OOD task, \( Z \in \{0, 1, 2\} \), \( X \in \mathbb{R}^3 \), and \( Y \in \{0, 1\} \). The training datasets have different distributions of \( P(Z) \), namely \( [0, 0.7, 0.3], [0.8, 0.1, 0.1], \text{ and } [0.1, 0.0, 0.9] \), with the test distribution set to \( P(Z) = [0.7, 0.2, 0.1] \). These tasks assess our method’s robustness to significant shifts in \( Z \). For both tasks, the conditional distribution \( P(Y \mid X, Z) \) is modeled as a Bernoulli distribution, where the parameter is a function of \( X \) and \( Z \), with significantly different values for each \( Z \). Further details are provided in Appendix~\ref{appendix:synthetic_data}.

\paragraph{Performance on Synthetic Data}

\begin{table}[ht]
    \vspace{-10pt}
    \centering
    \caption{Performance on Synthetic Data}
    \label{tab:synthetic_data}
    \resizebox{\columnwidth}{!}{
    \begin{tabular}{lrr}
        \toprule
        \multirow{2}{*}{Method} & \multicolumn{1}{c}{Proxy OOD Task} & \multicolumn{1}{c}{Multi-source OOD Task} \\
        \cmidrule(lr){2-2} \cmidrule(lr){3-3}
         & ACC OOD $\uparrow$ & ACC OOD $\uparrow$ \\
        \midrule
        ERM & 0.484 (0.001) & 0.466 (0.006) \\
        GroupDRO & 0.465 (0.011)  & 0.443 (0.016) \\
        IRM & 0.465 (0.001)  & 0.453 (0.005)  \\
        DeepCORAL & 0.463 (0.010) &  0.464 (0.005)\\
        VREx & 0.523 (0.034) & 0.490 (0.027)  \\
        DANN & 0.467 (0.007)  & 0.459 (0.005) \\
        ProxyDA & 0.485 (0.047)  & 0.462 (0.012)  \\
        Ours (\(\lambda = 0\)) & 0.870 (0.021) & \textbf{0.794 (0.037)} \\
        Ours & \textbf{0.896 (0.002)} & \textbf{0.811 (0.039)} \\
        \bottomrule
    \end{tabular}}
    \vspace{-10pt}
\end{table}

\begin{table*}[ht]
    \centering
    \caption{Performance on Real Data (Proxy and Multi-source OOD tasks)}
    \label{tab:real_data}
    \begin{tabular}{lccccc}
        \toprule
        \multirow{2}{*}{Method} & Employment & Income & Mobility & Public Coverage & Travel Time \\
        \cmidrule(lr){2-2} \cmidrule(lr){3-3} \cmidrule(lr){4-4} \cmidrule(lr){5-5} \cmidrule(lr){6-6}
         & ACC OOD $\uparrow$ & ACC OOD $\uparrow$ & ACC OOD $\uparrow$ & ACC OOD $\uparrow$ & ACC OOD $\uparrow$ \\
        \midrule
        ERM & 0.670 (0.005) & 0.805 (0.008) & 0.755 (0.002) & 0.776 (0.006) & 0.552 (0.024) \\
        GroupDRO & 0.654 (0.014) & 0.807 (0.026) & 0.745 (0.008) & 0.762 (0.004) & 0.492 (0.023) \\
        IRM & 0.663 (0.006) & 0.819 (0.022) & 0.751 (0.002) & 0.780 (0.002) & 0.562 (0.000) \\
        DeepCORAL & 0.664 (0.007) & 0.776 (0.021) & 0.741 (0.000) & 0.752 (0.031) & 0.526 (0.040) \\
        VREx & \textbf{0.702 (0.014)} & 0.851 (0.008) & 0.739 (0.018) & 0.758 (0.014) & 0.492 (0.023) \\
        DANN & 0.663 (0.007) & 0.773 (0.031) & 0.750 (0.005) & 0.779 (0.000) & 0.559 (0.010) \\
        ProxyDA & 0.689 (0.006) & 0.856 (0.040) & 0.657 (0.073) & NaN & 0.559 (0.039) \\
        Ours & \textbf{0.709 (0.005)} & \textbf{0.883 (0.001)} & \textbf{0.758 (0.008)} & \textbf{0.809 (0.003)} & \textbf{0.658 (0.001)} \\
        \bottomrule
    \end{tabular}
\end{table*}

Table \ref{tab:synthetic_data} shows the comparative performance on synthetic datasets. We report test accuracy on out-of-distribution tasks, presenting the mean accuracy over five runs along with the standard deviation. Our method significantly outperforms all baselines on both datasets. This is because most baselines rely on a single predictor, which cannot capture the heterogeneity in \( P(Y \mid X) \). For ProxyDA, since only one of the proxy or multiple sources is available, it is unable to effectively adapt to the shift. Additionally, we observe an improvement in performance with the regularizer.

\paragraph{Real Data}
For real data, we use datasets from \citet{ding2021retiring}, covering prediction tasks related to income, employment, health, transportation, and housing. These datasets span multiple years and all states of the United States, allowing researchers to study temporal shifts and geographic variations~\citep{liu2024need}. The input (\( X \)) includes features such as education, occupation, marital status, and demographic attributes like sex and race.

\paragraph{Proxy OOD task}
For the Proxy OOD task, we evaluate on the ACS Employment, ACS Public Coverage, and ACS Travel Time datasets. We consider a latent confounder ($Z$) that varies across datasets, with $Z = 1$ representing the minority group and $Z = 0$ representing the majority group. The distribution of $Z$ during training is $\Ptr{Z = 0} = 0.95$ and $\Ptr{Z = 1} = 0.05$, while the test distribution flips to $\Pte{Z = 0} = 0.05$ and $\Pte{Z = 1} = 0.95$. The specific datasets are:

\begin{itemize}
    \item \textbf{ACS Employment:} This dataset predicts employment status. We consider disability status as the unobserved confounder. Proxies such as public insurance coverage and independent living status ($S$) are used to infer disability.

    \item \textbf{ACS Public Coverage:} This dataset predicts whether an individual is covered by public health insurance. We consider 'age' as the unobserved confounder and use a synthetic proxy variable.

    \item \textbf{ACS Travel Time:} This dataset predicts whether an individual has a commute to work that is longer than 20 minutes. We consider 'state' as the unobserved confounder and use a synthetic proxy variable.
\end{itemize}

\paragraph{Multi-source OOD task}
For the Multi-source OOD task, we evaluate on the ACS Income and ACS Mobility datasets. We consider the state of residence ($Z$) as the latent confounder across both datasets, with varying distributions between training and testing sets. The specific datasets are:

\begin{itemize}
    \item \textbf{ACS Income:} This dataset predicts individual income. We train on three synthetic splits: 20\% Puerto Rico (PR) and 80\% California (CA); 20\% South Dakota (SD) and 80\% California (CA); and 20\% PR, 20\% SD, and 60\% CA. Testing is performed on a split of 90\% PR, 5\% SD, and 5\% CA.

    \item \textbf{ACS Mobility:} This dataset predicts whether an individual had the same residential address one year ago. We consider 'state' as the unobserved confounder, similar to ACS Income. We use the same training and testing distributions as ACS Income.
\end{itemize}

Full details for the datasets are provided in Appendix~\ref{appendix:real_data}.

\paragraph{Performance on Real Data}

Table \ref{tab:real_data} shows the comparative performance on the five datasets. We report test accuracy on out-of-distribution tasks, presenting the mean accuracy over five runs along with the standard deviation. We provide AUC scores in Appendix~\ref{appendix:additional_auc}. Notably, our method outperforms all baselines in both datasets, demonstrating its robust ability to handle shifts in $P(X, Y)$ caused by latent confounder shifts. For ACS Public Coverage, we were unable to compare with ProxyDA as our data does not meet the necessary full-rank conditions. We discuss these conditions for ProxyDA in Appendix~\ref{appendix:assumption_discussion}.

\paragraph{Scalability}

\begin{figure}[h]
    \centering
    \resizebox{\linewidth}{!}{\input{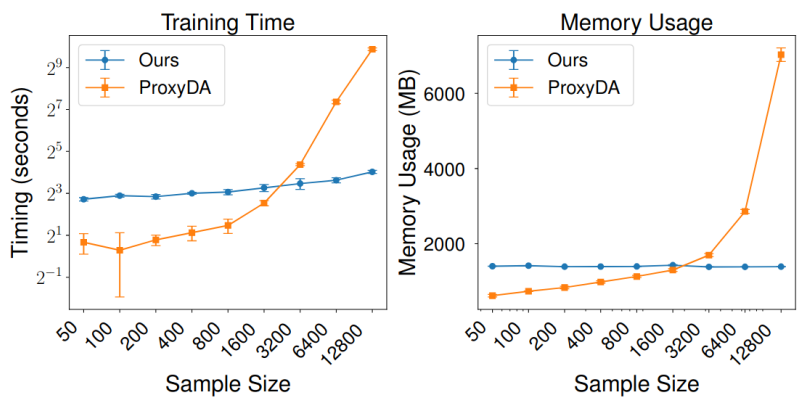}}
    \caption{Training time and memory usage for Our method and ProxyDA}
    \label{fig:timeandmemory}
\end{figure}

To evaluate the scalability of our approach, we measured memory consumption and runtime during training on the Proxy task with synthetic data for different sample sizes. Figure~\ref{fig:timeandmemory} compares our method with ProxyDA, showing peak memory usage and runtime as functions of sample size. ProxyDA, as expected, shows cubic time complexity, typical of kernel methods. In contrast, our approach theoretically has linear time complexity for a fixed number of iterations. In practice, early stopping reduces the number of iterations for larger sample sizes, thereby lowering runtime further as seen in the plot. Regarding memory, our method's memory usage stays constant while ProxyDA's memory usage increases with sample size. This makes it infeasible for ProxyDA to scale to datasets with large sample sizes.

\section{CONCLUSION} \label{sec:conclusions}
In this work, we tackled the challenging task of out-of-distribution generalization in the presence of unobserved confounders. We introduced a simple yet scalable method that first estimates the distribution of latent confounder and then adapts the classifier to the test distribution. Our approach not only theoretically proves the identifiability of the latent distribution but also empirically outperforms existing baselines across multiple tasks. This demonstrates the robustness and practical effectiveness of our method, providing a strong foundation for future research in developing OOD-robust techniques that handle unobserved confounders.

While our method operates under relatively simple assumptions, there may be scenarios where these assumptions fail to hold, such as when the confounder is not discrete. Future work should consider understanding and identifying these limitations and build upon our proposed method.

\subsubsection*{Acknowledgments}
The authors would like to thank the anonymous reviewers for valuable comments and suggestions. We would also like to acknowledge the support from NSF awards IIS-2340124, CCF-1918483, CAREER IIS-1943364 and CNS-2212160 and NIH grant U54HG012510.

\renewcommand{\bibsection}{\subsubsection*{References}}
\bibliographystyle{plainnat}
\bibliography{ref}

\appendix
\onecolumn
\aistatstitle{Supplementary Materials}
\section{PROOFS}
\label{appendix:proofs}
\subsection{Proof of Theorem 1}
\label{appendix:approximate_identifiability_proof}
\renewcommand{\thetheorem}{\ref{theorem:identifiability}}
\begin{theorem}[Approximate Identifiability of Latent Variables]
Let $P(X, S, Z)$ be the true distribution and $Q(X, S, Z)$ be any other distribution over the observed variables $X$, proxy variables $S$, and latent variables $Z$. Suppose that the marginal distributions over $X$ and $S$ are identical, i.e., $P(X, S) = Q(X, S)$, and that Assumptions~\ref{asp:structural_constraints}-~\ref{asp:weak_overlap} are satisfied for both $P$ and $Q$. Then, there exists a permutation $\pi$ of $\{1, 2, \dots, n_z\}$ such that for each $i \in \{1, 2, \dots, n_z\}$,  conditional distributions $P(Z \mid X = x)$ and $Q(Z \mid X = x)$ satisfy the following bound:

\[
\sup_{x \in \text{supp}(X)} \left| P(Z = i \mid X = x) - Q(Z = \pi(i) \mid X = x) \right| \leq O\left( \frac{1 - \eta}{2\eta - 1} \right)
\]

where $\eta \in \left( \frac{1}{2}, 1 \right]$ is as defined in Assumption \ref{asp:weak_overlap}.

\end{theorem}
\renewcommand{\thetheorem}{\arabic{theorem}}

\begin{proof}
The proof follows three main steps:
\begin{enumerate}
    \item We first establish there is a unique value of \(n_z\) that is compatible with a given observed distribution \(P(X, S)\).
    \item We show that the conditional distributions $P(Z \mid X = x)$ and $Q(Z \mid X = x)$ are related through a linear transformation. We further show that this linear transformation must satisfy certain constraints.
    \item Finally, using Assumption~\ref{asp:weak_overlap}, which limits the overlap between latent variables, we show that $A$ can be approximated by a permutation matrix with an error bounded by $O\left( \frac{1 - \eta}{2\eta - 1} \right)$. This establishes a bound on the difference between $P(Z \mid X = x)$ and $Q(Z \mid X = x)$.

\end{enumerate}
\paragraph{1. Establishing Dimension of the Subspace Spanned by $P(S \mid X = x)$ = $n_z$}\textbf{ }\\

Under Assumption~\ref{asp:structural_constraints}, we have that $S \perp X \mid Z$. Therefore, the conditional distribution can be expressed as:
\[
P(S \mid X = x) = \sum_{z=1}^{n_z} P(S \mid Z = z) P(Z = z \mid X = x).
\]
Let us consider $P(S \mid Z)$ as a matrix $M_P \in \mathbb{R}^{n_z \times n_s}$ where each row corresponds to $P(S \mid Z = z)$ for $z \in \{1, 2, \dots, n_z\}$. By Assumption~\ref{asp:proxy_rank}, $\text{rank}(M_P) = n_z$.
\\
\\
Since $P(S \mid X = x)$ is a linear combination of the rows of $M_P$, the set of all such vectors $\{ P(S \mid X = x) \}_{x \in \text{supp}(X)}$ lies within the row space of $M_P$.
\\
\\
Given that $M_P$ has full rank, the dimension of this subspace is exactly $n_z$. Formally, the span of $\{ P(S \mid X = x) \}$ satisfies:
\[
\text{dim}\left( \text{span} \left\{ P(S \mid X = x) \mid x \in \text{supp}(X) \right\} \right) = n_z.
\]
Therefore, there is a unique \(n_z\) for a given observed distribution.
\\
\paragraph{2. Relating $P(Z \mid X = x)$ and $Q(Z \mid X = x)$ through a Linear Transformation}\textbf{ }\\
\\
Given that $P(X, S) = Q(X, S)$, it follows that for all $x \in \text{supp}(X)$:
\[
P(S \mid X = x) = Q(S \mid X = x).
\]
Using Assumption~\ref{asp:structural_constraints}, we can decompose these conditional distributions as:
\begin{align}
P(S \mid X = x) &= \sum_{z=1}^{n_z} P(S \mid Z = z) P(Z = z \mid X = x) = M_P^T P(Z \mid X = x), \label{eq:P_decomp} \\
Q(S \mid X = x) &= \sum_{z=1}^{n_z} Q(S \mid Z = z) Q(Z = z \mid X = x) = M_Q^T Q(Z \mid X = x), \label{eq:Q_decomp}
\end{align}
where we define the matrices $M_P \in \mathbb{R}^{n_z \times n_s}$ and $M_Q \in \mathbb{R}^{n_z \times n_s}$ such that:
\[
(M_P)_{z,s} = P(S = s \mid Z = z), \quad (M_Q)_{z,s} = Q(S = s \mid Z = z).
\]
Given that $P(S \mid X = x) = Q(S \mid X = x)$ for all $x$, we have:
\[
M_P^T P(Z \mid X = x) = M_Q^T Q(Z \mid X = x).
\]
Since Assumption~\ref{asp:proxy_rank} ensures that $M_P$ has full rank and that $n_s \geq n_z$, the matrix $M_P^\top$ admits a left Moore-Penrose pseudo-inverse, denoted by $M_P^{\dagger} \in \mathbb{R}^{n_z \times n_s}$, satisfying:
\[
M_P^{\dagger} M_P^\top = I_{n_z},
\]
where $I_{n_z}$ is the $n_z \times n_z$ identity matrix. Multiplying both sides of the equation by $M_P^{\dagger}$ yields:
\[
M_P^{\dagger} M_P^\top P(Z \mid X = x) = M_P^{\dagger} M_Q^\top Q(Z \mid X = x),
\]
which simplifies to:
\begin{equation}
P(Z \mid X = x) = A Q(Z \mid X = x), \label{eq:linear_relation}
\end{equation}
where we define the transformation matrix $A$ as:
\[
A = M_P^{\dagger} M_Q^\top.
\]

\textbf{Constraints:}
We show that \(A\) must satisfy certain constraints and cannot be an arbitrary matrix.

1)We claim that $\mathbf{1}^\top A = \mathbf{1}^\top$, where $\mathbf{1} \in \mathbb{R}^{n_z}$ is the vector of ones. In other words, the sum of elements in each column of $A$ is 1.

Proof: Given that $P(S \mid X = x) = Q(S \mid X = x)$ for all $x$, the columns of $M_P^\top$ and $M_Q^\top$ span the same subspace. Therefore, there exists a matrix $C \in \mathbb{R}^{n_z \times n_z}$ such that:
\[
M_P^\top C = M_Q^\top.
\]
This implies,
\[
C = M_P^{\dagger} M_Q^\top = A
\]
We know that both $M_P$ and $M_Q$ represent probability distributions \(P(S \mid Z)\) and \(Q(S \mid Z)\), meaning that:
\[
\mathbf{1}^\top M_P^\top = \mathbf{1}^\top, \quad \mathbf{1}^\top M_Q^\top = \mathbf{1}^\top,
\]
where $\mathbf{1} \in \mathbb{R}^{n_s}$ is the vector of ones. Applying these equalities to the equation $M_P^\top A = M_Q^\top$, we get:
\[
\mathbf{1}^\top M_P^\top A = \mathbf{1}^\top M_Q^\top.
\]
Substituting the known equalities:
\begin{equation}
\mathbf{1}^\top A = \mathbf{1}^\top, \label{equation:sumcolumn1}
\end{equation}

This means that the sum of the elements in each column of $A$ is 1.

2)We derive upper and lower bounds on the entries of \(A\).

Consider any row $r$ of $A$. By the definition of conditional probabilities and equation \eqref{eq:linear_relation}, we have:
\[
P(Z = r \mid X = x) = \sum_{i=1}^{n_z} A_{r i} Q(Z = i \mid X = x), \quad \forall x \in \text{supp}(X).
\]
Since $P(Z = r \mid X = x)$ is a valid probability, it satisfies:
\[
0 \leq \sum_{i=1}^{n_z} A_{r i} Q(Z = i \mid X = x) \leq 1.
\]
Subtracting $\frac{1}{2}$ from both sides, we obtain:
\[
-\frac{1}{2} \leq \sum_{i=1}^{n_z} \left( A_{r i} - \frac{1}{2} \right) Q(Z = i \mid X = x) \leq \frac{1}{2}.
\]
Let $\lambda = \max_{i} \left| A_{r i} - \frac{1}{2} \right|$. Without loss of generality, assume that the maximum is achieved at $i = 1$. Then:
\[
\left| \sum_{i=1}^{n_z} \left( A_{r i} - \frac{1}{2} \right) Q(Z = i \mid X = x) \right| \leq \frac{1}{2}.
\]
Applying the triangle inequality, we have:
\[
\left| \left( A_{r1} - \frac{1}{2} \right) Q(Z = 1 \mid X = x) \right| - \left| \sum_{i=2}^{n_z} \left( A_{r i} - \frac{1}{2} \right) Q(Z = i \mid X = x) \right| \leq \frac{1}{2}.
\]
Since $|A_{r i} - \frac{1}{2}| \leq \lambda$ for all $i$, it follows that:
\[
\lambda Q(Z = 1 \mid X = x) - \lambda \sum_{i=2}^{n_z} Q(Z = i \mid X = x) \leq \frac{1}{2}.
\]
Simplifying, we get:
\[
\lambda Q(Z = 1 \mid X = x) - \lambda (1 - Q(Z = 1 \mid X = x)) \leq \frac{1}{2},
\]
which further simplifies to:
\[
\lambda (2 Q(Z = 1 \mid X = x) - 1) \leq \frac{1}{2}.
\]
Given that there exists $x^* \in \text{supp}(X)$ such that $Q(Z = 1 \mid X = x^*) \geq \eta$, where $\eta \in \left( \frac{1}{2}, 1 \right]$, we substitute to obtain:
\[
\lambda (2 \eta - 1) \leq \frac{1}{2} \quad \Rightarrow \quad \lambda \leq \frac{1}{2 (2 \eta - 1)}.
\]
Therefore, each entry $A_{r i}$ satisfies:
\begin{equation}
-\frac{1 - \eta}{2 \eta - 1} \leq A_{r i} \leq 1 + \frac{1 - \eta}{2 \eta - 1}. \label{eq:A_bounds}
\end{equation}

3)We derive lower bounds on the maximum entry in each row of \(A\)

Let $\mu = \max_{i} A_{r i}$. From the definition of $P(Z = r \mid X = x)$, we have:
\[
P(Z = r \mid X = x) = \sum_{i=1}^{n_z} A_{r i} Q(Z = i \mid X = x) \leq \mu \sum_{i=1}^{n_z} Q(Z = i \mid X = x) = \mu.
\]
Since $\max_{x} P(Z = r \mid X = x) \geq \eta$, it follows that $\mu \geq \eta$.

\paragraph{3. \(A\) can be approximated by a permutation matrix}\textbf{ }\\
\\

We aim to show that there exists a permutation matrix $\Pi$ such that:
\[
|A_{i j} - \Pi_{i j}| < O\left( \frac{1 - \eta}{2 \eta - 1} \right), \quad \forall i, j.
\]
If \(n_z = 1\), A is trivially the identity matrix and is equal to a permutation matrix. Henceforth, we assume \(n_z \geq 2\).
Consider the following cases:

\textbf{Case 1: $\frac{1}{2} < \eta \leq \frac{n_z + 1}{n_z + 2}$}
For this case, consider any generic permutation matrix \(\Pi\).
From equation \eqref{eq:A_bounds}, we have:
\[
|A_{i j} - \Pi_{i j}| \leq 1 + \frac{1 - \eta}{2 \eta - 1} \leq \frac{n_z + 2}{n_z + 2} + \frac{1 - \eta}{2 \eta - 1} \leq (n_z + 2)(1 - \eta) + \frac{1 - \eta}{2 \eta - 1} \leq \frac{(n_z + 2)(1 - \eta) }{2\eta - 1} + \frac{1 - \eta}{2 \eta - 1} = O\left( \frac{1 - \eta}{2 \eta - 1} \right).
\]

\textbf{Case 2: $\frac{n_z + 1}{n_z + 2} < \eta \leq 1 $}

Starting with the given inequality $\eta > \frac{n_z + 1}{n_z + 2}$, we have
\begin{align*}
\eta (n_z + 2) &> n_z + 1 \\
n_z \eta - n_z &> 1 - 2 \eta \\
n_z (1 - \eta) &< 2 \eta - 1 \\
(2 \eta - 1) &> (n_z - 2)\frac{(1 - \eta)}{2\eta - 1}\\
2 \eta &> 1 + (n_z - 2)\frac{(1 - \eta)}{2\eta - 1}
\label{equation:bounds}
\end{align*}

The above inequality, equation \eqref{equation:sumcolumn1} and \eqref{eq:A_bounds} imply that no more than one element in any column can be greater than or equal to \(\eta\).

Since each row must have at least one element \(\geq \eta\) and \(A\) is a square matrix, it follows that each column has exactly one element \(\geq \eta\). Without loss of generality, consider that for column \(c\), the first element is \(\geq \eta\), i.e. \(A_{1c} \geq \eta\).

Then,
\[\sum_{i=2}^{n_z} A_{ic} = 1 - A_{1c}\]
Since all \(A_{ic} \geq -\frac{1-\eta}{2\eta - 1}\), we have:
\[
\max_{i \neq 1} A_{ic} \leq (1-\eta) + (n_z - 2)\frac{1-\eta}{2\eta - 1}
\]

Also,
\[|1 - A_{1c}| \leq \max \left\{ 1-\eta, \frac{1-\eta}{2\eta - 1} \right\}\]
and
\[|A_{ic}| \leq \max \left\{ (1-\eta) + (n_z-2)\frac{1-\eta}{2\eta-1}, \frac{1-\eta}{2\eta-1} \right\}\]

Define the permutation matrix \(\Pi\) such that:
\[
\Pi_{i j} =
\begin{cases}
1 & \text{if } A_{i j} \geq \eta, \\
0 & \text{otherwise}.
\end{cases}
\]
By construction, $\Pi$ satisfies:
\[
|\Pi_{i j} - A_{i j}| \leq \max\left\{ 1 - \eta, \frac{(n_z - 2)(1 - \eta)}{2 \eta - 1}, \frac{(1 - \eta)}{2 \eta - 1}\right\} = O\left( \frac{1 - \eta}{2 \eta - 1} \right).
\]

We now obtain the final bound on the difference of the posterior distributions.

From equation \eqref{eq:linear_relation}, we have $P(Z|X=x) = A Q(Z|X=x)$. Using this, we can derive:

\begin{align*}
|P(Z = i|X = x) - Q(Z=\pi(i)|X=x)| &= \left|\sum_{j=1}^{n_z} A_{ij}Q(Z=j|X=x) - Q(Z=\pi(i)|X=x)\right| \\
&= \left|\sum_{j=1}^{n_z} (A_{ij} - \Pi_{ij})Q(Z=j|X=x)\right| \\
&\leq \sum_{j=1}^{n_z} |A_{ij} - \Pi_{ij}| |Q(Z=j|X=x)| \\
&\leq \max_{j} |A_{ij} - \Pi_{ij}| \sum_{j=1}^{n_z} |Q(Z=j|X=x)| \\
&= \max_{j} |A_{ij} - \Pi_{ij}|
\end{align*}

where we used the triangle inequality and the fact that $\sum_{j=1}^{n_z} |Q(Z=j|X=x)| = 1$ since $Q(Z|X=x)$ is a probability distribution.

From our analysis in step 6, we know that:

\[
\max_{j} |A_{ij} - \Pi_{ij}| \leq O\left(\frac{1-\eta}{2\eta-1}\right)
\]

Therefore, we can conclude:

\[
\sup_{x \in \text{supp}(X)} \left| P(Z = i \mid X = x) - Q(Z = \pi(i) \mid X = x) \right| \leq O\left( \frac{1 - \eta}{2\eta - 1} \right),
\]

for a permutation \(\pi\) corresponding to matrix \(\Pi\)

which proves Theorem \ref{theorem:identifiability}.
\end{proof}

\subsection{Proof of Corollary 1}
\renewcommand{\thecorollary}{\ref{corollary:full_identifiability}}
\begin{corollary}[Full Identifiability]

If $\eta = 1$, then the latent variables are fully identifiable. That is, there exists a permutation $\pi$ such that, for all $x \in \text{supp}(X)$ and for each $i \in \{1, 2, \dots, n_z\}$,

\[
P(Z = i \mid X = x) = Q(Z = \pi(i) \mid X = x).
\]

\end{corollary}

\renewcommand{\thecorollary}{\arabic{corollary}}
\begin{proof}
The proof follows directly from Theorem \ref{theorem:identifiability}. When $\eta = 1$, the bound in equation \eqref{eq:bound} becomes:

\begin{align*}
\sup_{x \in \text{supp}(X)} \left| P(Z = i \mid X = x) - Q(Z = \pi(i) \mid X = x) \right|
&\leq O\left( \frac{1 - \eta}{2\eta - 1} \right) \\
&= O\left( \frac{1 - 1}{2(1) - 1} \right) \\
&= O(0) \\
&= 0
\end{align*}
Therefore,
\[
P(Z = i \mid X = x) = Q(Z = \pi(i) \mid X = x)
\]

for all $x$ and $i$, proving that the latent variables are fully identifiable up to permutation when $\eta = 1$.
\end{proof}

\subsection{Proof of Proposition 1}
\label{appendix:proposition_proof}
This proposition and proof closely follows \citet{d2021overlap}.
\renewcommand{\theproposition}{\ref{prop:eta}}
\begin{proposition}
    Let \( Z \in [n_z] \) and \( X \in \mathbb{R}^{dim(X)} \). Further, assume that each feature \( X_{i} \) is sufficiently discriminative of the latent values of \( Z \), conditioned on all previous features \( X_{1:i-1} \). Specifically, for all feature indices \( i \in \{1,...,dim(X)\} \) and all latent values \( j \in [n_z] \), there exists \(\epsilon > 0\) such that, for all \( k \in [n_z] \) where \( k \neq j \),
    \begin{small}
    \(\mathbb{E}_{P(X|Z=j)}\text{KL}\left(P(X_i|X_{1:i-1},Z=j) \parallel P(X_i|X_{1:i-1}, Z=k)\right)
    \)
    \end{small} is greater than \(\epsilon\).
    Then, as \( dim(X) \rightarrow \infty \), there exists no \( \eta \in (0,1) \) such that
    \[    \max_{x} P(Z = j \mid X = x) \leq \eta \quad \forall j\]
\end{proposition}

\renewcommand{\theproposition}{\arabic{proposition}}

\begin{proof}
For simplicity, we assume $Z \in \{0,1\}$. We proceed by contradiction.

Suppose there exists an $\eta \in (0,1)$ such that for all $x$:
\[P(Z = 0 \mid X = x) \leq \eta \quad \text{and} \quad P(Z = 1 \mid X = x) \leq \eta\]

This implies that for all x:
\[1 - \eta \leq P(Z = 0 \mid X = x) \leq \eta \quad \text{and} \quad 1 - \eta \leq P(Z = 1 \mid X = x) \leq \eta\]
It follows that $P(Z = 0) > 0$ and $P(Z = 1) > 0$. Let $\gamma = P(Z = 1) / P(Z = 0)$.
For any $x$:
\[\frac{P(X = x \mid Z = 0)}{P(X = x \mid Z = 1)} = \frac{P(Z = 0 \mid X = x)}{P(Z = 1 \mid X = x)} \cdot \frac{1}{\gamma} \leq \frac{\eta}{1-\eta} \cdot \frac{1}{\gamma}\]
Taking logarithms:
\[\log \frac{P(X = x \mid Z = 0)}{P(X = x \mid Z = 1)} \leq \log \frac{\eta}{1-\eta} - \log \gamma\]
Now, consider the KL divergence between $P(X \mid Z = 0)$ and $P(X \mid Z = 1)$:
\begin{align*}
&\mathrm{KL}(P(X \mid Z = 0) \parallel P(X \mid Z = 1)) \\
&= \mathbb{E}_{X \sim P(X \mid Z = 0)}\left[\log \frac{P(X \mid Z = 0)}{P(X \mid Z = 1)}\right] \\
&\leq \log \frac{\eta}{1-\eta} - \log \gamma
\end{align*}
This implies that the KL divergence is bounded.
However, by the chain rule of KL divergence and the condition given in the proposition:
\begin{align*}
&\mathrm{KL}(P(X \mid Z = 0) \parallel P(X \mid Z = 1)) \\
&= \sum_{i=1}^{\dim(X)} \mathbb{E}_{P(X \mid Z = 0)}\left[\mathrm{KL}(P(X_i \mid X_{1:i-1}, Z = 0) \parallel P(X_i \mid X_{1:i-1}, Z = 1))\right] \\
&\geq \dim(X) \cdot \epsilon
\end{align*}
As $\dim(X) \to \infty$, the right-hand side approaches infinity, contradicting the boundedness we established earlier.
Therefore, our initial assumption must be false. We conclude that as $\dim(X) \to \infty$, there exists no $\eta \in (0,1)$ such that $\max_x P(Z = z \mid X = x) \leq \eta$ for all $z \in \{0,1\}$ .
\end{proof}

\subsection{Proof of Theorem 2}
\label{appendix:regularization_proof}
\renewcommand{\thetheorem}{\ref{theorem:regularization}}
\begin{theorem}
Consider the set \( \mathcal{M} \), which denotes the set of all matrices \( M \) associated with some distribution \( P(S \mid Z) \) for which there exists a \(\phi(\cdot)\) associated with density \( \Ptr{Z \mid X} \) such that \( M \) and \(\phi(\cdot)\) are minimizers of the reconstruction loss in \Cref{eq:recloss}. Let \( M^* \in \mathcal{M} \) be the matrix associated with a \( Z \) that satisfies Assumption~\ref{asp:weak_overlap} for \(\eta = 1\). Then it holds that
\[
 \mathcal{L}_{\text{var}}(M^{*}) \leq \mathcal{L}_{\text{var}}(M), \quad \forall M \in \mathcal{M},
\]
where
\begin{equation}
\mathcal{L}_{\text{var}}(M) = \max_{z \in [n_z]} \left( \frac{1}{n_s} \sum_{s=1}^{n_s} \left(M_{zs} - \frac{1}{n_s}\right)^2 \right)
\end{equation}
\end{theorem}

\renewcommand{\thetheorem}{\arabic{theorem}}

\begin{proof}
    To prove this, we first prove the following lemma.
\begin{lemma}
Let \( \mathbf{v}, \mathbf{w}_1, \mathbf{w}_2, \ldots, \mathbf{w}_n \in [0,1]^k \) be vectors representing probability distributions, i.e., \( \sum_{j=1}^k w_{ij} = 1 \) for all \(1 \leq i \leq n\) and \( \sum_{j=1}^k v_{j} = 1 \).
Suppose \( \mathbf{v} = \sum_{i=1}^n \alpha_i \mathbf{w}_i \) where \(\sum_{i=1}^n \alpha_i = 1\) and \( \alpha_i \geq 0 \) for all \( i \). That is, \(\mathbf{v}\) is a convex combination of the vectors \( \mathbf{w}_i \). Let \(\text{Var}(\mathbf{v})\) denote the variance of this categorical distribution. Under these conditions, the variance of \( \mathbf{v} \) satisfies:

\[
\text{Var}(\mathbf{v}) \leq \sum_{i=1}^n \alpha_i \text{Var}(\mathbf{w}_i) \leq \max_{i} \text{Var}(\mathbf{w}_i).
\]
\end{lemma}
\begin{proof}

Recall the definition of variance for a vector \( \mathbf{v} \in \mathbb{R}^k \) is given by:

\[
\text{Var}(\mathbf{v}) = \frac{1}{k} \sum_{j=1}^k (v_j - \bar{v})^2,
\]

where \( v_j \) is the \( j \)-th component of \( \mathbf{v} \) and \( \bar{v} \) is the mean of the components of \( \mathbf{v} \). Given \( \mathbf{v} = \sum_{i=1}^n \alpha_i \mathbf{w}_i \), each component \( v_j \) of \( \mathbf{v} \) can be expressed as:

\[
v_j = \sum_{i=1}^n \alpha_i w_{ij}.
\]

The mean \( \bar{v} \) of the components of \( \mathbf{v} \) can be calculated as:

\[
\bar{v} = \frac{1}{k} \sum_{j=1}^k v_j = \frac{1}{k} \sum_{j=1}^k \sum_{i=1}^n \alpha_i w_{ij}.
\]

Since each \( \mathbf{w}_i \) sums to 1, it follows that:

\[
\bar{v} = \sum_{i=1}^n \alpha_i \left( \frac{1}{k} \sum_{j=1}^k w_{ij} \right) = \sum_{i=1}^n \alpha_i \bar{w_i}.
\]

Now, substituting back into the variance formula, we get:

\[
\text{Var}(\mathbf{v}) = \frac{1}{k} \sum_{j=1}^k \left(\sum_{i=1}^n \alpha_i (w_{ij} - \bar{w_i})\right)^2.
\]

Applying the AM-RMS inequality:

\[
\left(\sum_{i=1}^n \alpha_i (w_{ij} - \bar{w_i})\right)^2 \leq \left(\sum_{i=1}^n \alpha_i (w_{ij} - \bar{w_i})^2\right),
\]
and since $\alpha_i^2 \leq \alpha_i \in [0,1]$,
we can simplify the variance expression:

\[
\text{Var}(\mathbf{v}) \leq \sum_{i=1}^n \alpha_i \frac{1}{k} \sum_{j=1}^k (w_{ij} - \bar{w_i})^2.
\]

Finally, as \( \frac{1}{k} \sum_{j=1}^k (w_{ij} - \bar{w_i})^2 \) is the expression for the variance of \( \mathbf{w}_i \), we conclude:

\[
\text{Var}(\mathbf{v}) \leq \sum_{i=1}^n \alpha_i \text{Var}(\mathbf{w}_i) \leq \max_{i} \text{Var}(\mathbf{w}_i).
\]
\end{proof}

Now, we discuss the proof of Theorem~\ref{theorem:regularization}.

Since \((\phi^*(.), M^*)\) satisfies the weak overlap Assumption~\ref{asp:weak_overlap} for \(\eta = 1\), we know that \(\exists x^{*}\) such that \(\phi^{*}_i(x^{*}) = 1 \) for all \( i, 1 \leq i \leq n_z \) (because \(\phi(x) = P(Z \mid X=x)\)). Therefore, we know that
\[
\forall i, \exists x^{*} \text{ such that } P(S \mid X=x^{*}) = \sum_z \phi^{*}_z(x^{*}) M^{*}_{z,:} = M^{*}_{i,:} \,.
\]
Moreover, for any valid $\phi_z(x^{*})$,
\[
P(S \mid X=x^{*}) = \sum_z \phi_z(x^{*}) M_{z,:}\, .
\]
Therefore, for all i,
\[
M^{*}_{i,:} = \sum_z \phi_z(x^{*}) M_{z,:}\, .
\]
Hence, using Lemma 1, for all $i$ we have
\[
\text{Var}(M^{*}_{i,:}) \leq \max_z \text{Var}(M_{z,:}) \, .
\]
Hence,
\[
\max_z \text{Var}(M^{*}_{z,:}) \leq \max_z \text{Var}(M_{z,:}).
\]
Since $M^{*}_{z,:}$ and $M_{z,:} \in \mathbb{R}^{n_s}$, the mean of these is $\frac{1}{n_s}$.
Hence, \[
\max_z \text{Var}(M_{z,:}) =
\mathcal{L}_{\text{var}}(M) = \max_{z \in \{1, \dots, n_z\}} \left( \frac{1}{n_s} \sum_{s=1}^{n_s} (M_{s,z} - \frac{1}{n_s})^2 \right),
\]
\[
\max_z \text{Var}(M^{*}_{z,:}) =
\mathcal{L}_{\text{var}}(M^{*}) = \max_{z \in \{1, \dots, n_z\}} \left( \frac{1}{n_s} \sum_{s=1}^{n_s} (M^{*}_{s,z} - \frac{1}{n_s})^2 \right).
\]
Therefore,
\[
\mathcal{L}_{\text{var}}(M^{*}) \leq \mathcal{L}_{\text{var}}(M).
\]
\end{proof}

\section{DETAILED DISCUSSION ON ASSUMPTIONS}
\label{appendix:assumption_discussion}
In this section, we comprehensively discuss our assumptions and compare them to those in our closest related work, Proxy Domain Adaptation (ProxyDA)~\citep{ProxyDA}. We then examine the practicality of these assumptions. Note that ProxyDA operates under two settings: one with multiple sources and a proxy variable, and the other with a concept variable and a proxy variable. Since our model is designed to work with either multiple sources or a proxy variable, we will only compare it with the former setting.

\begin{table}[h!]
\centering
\resizebox{\columnwidth}{!}{
\begin{tabular}{@{}lll@{}}
\toprule
\textbf{Assumptions} & \textbf{Ours} & \textbf{Proxy Domain Adaptation} \\ \midrule
Task & Out-of-distribution & Domain Adaptation \\
Observed variables & $X$, $Y$ and ($U$ or $S$) & $X$, $Y$, $U$ and $S$ \\
Number of labeled sources & 1 & multiple ($\geq n_z$) \\
Conditional independence & $S \perp X,Y \mid Z$ (or $U \perp X,Y \mid Z$) & $U \perp X,Y,S \mid Z$ and $S \perp X \mid Z$ \\
Causality between $X$ and $Y$ & $X \leftrightarrow Y$ & $X \rightarrow Y$\\
Full rank & $P(S \mid Z)$ (or $P(U \mid Z)$) must be full rank & $P(S \mid Z)$ must be full rank and $P(Z \mid U,X)$ must be full rank for X \\
Discrete unobserved confounder & Yes & Yes \\
Informative variable & No & Yes \\
Existence of high-likelihood sample & Yes & No \\
Proxy available at test time & No & Yes \\
$ns \geq n_z$ & Yes & Yes \\
$n_z$ known & No & No \\

\bottomrule
\end{tabular} }
\caption{Comparison of Assumptions between Ours and \citet{ProxyDA}}
\label{tab:assumptions_comparison}
\end{table}

\paragraph{Our assumptions}
(i) In real-world tasks, test data is rarely available during training and train data is rarely available during inference. Hence, the OOD task we consider is more realistic and challenging. (ii) We only require \( U \) (multiple training sources) or \( S \), a proxy for the confounder, not both simultaneously. Additionally, we do not require a proxy at test time, relaxing previous assumptions. Moreover, for the multi-source task, only one of the sources needs to be labeled. (iii) Unlike ProxyDA, we allow for a bi-directional arrow between \(X\) and \(Y\) in our causal diagram. The causal relationship between \( X \) and \( Y \) can vary widely based on the problem. For example, in a house price prediction problem, it is likely that \( X \rightarrow Y \) since features of the house likely causally affect the price of the house. However, in image classification problems, it is common to have \( Y \rightarrow X \), where the category of the object causes the distribution of pixels in the image. Furthermore, in many cases, some features in \( X \) might cause \( Y \) and others may be effects of \( Y \). For example, in our income prediction task, it is likely that education causes income but health insurance coverage is caused by income. Both are features in \( X \). Therefore, allowing a bidirectional causal relationship between \( X \) and \( Y \) increases the scope of our method. (iv) We are slightly more restrictive regarding conditional independence. While ProxyDA allows \( S \) to cause \( Y \), we do not allow that relation. If \(S\) causes \(Y\), we cannot allow \(Y\) to cause \(X\). This is because if \( S \) causes \( Y \) and \( Y \) causes \( X \), then \( S \perp X \mid Z \) does not hold. Therefore, there is an inherent trade-off between (iii) and (iv).

(v)As far as we know, all latent shift methods \citep{Alabdulmohsin2023, ProxyDA} make full rank assumptions. Full rank assumptions ensure diversity in distributions with respect to the confounder and are necessary to prove identifiability. However, we do not require any full rank condition on \( P(Z \mid S,X) \) or \( P(Z \mid X) \) for all \( X \), which could be quite strict. For example, the full rank assumption on \( P(Z \mid S,X) \) for all \( X \) does not allow \( Z \) to be deterministic for any \( X \).

(vi)While the informative variable assumption is common in the literature \citep{Miao, ProxyDA}, we do not require this assumption. This assumption does not allow certain distributions, for example, those where \( P(Z \mid X) \) can take only a finite set of values.

One of the key assumptions we make, which other papers do not, is the weak overlap assumption. Our approximate identifiability result in Theorem~\ref{theorem:identifiability} shows that we can bound the difference between the posteriors of any two possible distributions by \(O(\frac{1 - \eta}{2\eta - 1})\). Proposition~\ref{prop:eta} shows that Assumption~\ref{asp:weak_overlap} would hold for an \(\eta\) arbitrarily close to 1 for high-dimensional data. Therefore, given the observed distribution \(P(X,S)\), we can approximate the latent distribution by an arbitrarily small error with the increase in dimension of \(X\).

\section{EXPERIMENTAL DETAILS} \label{app:exp}
In this section, we provide additional details on our dataset construction, baselines, model implementation, and model selection.
\subsection{Synthetic Data}
\label{appendix:synthetic_data}

\paragraph{Proxy OOD Task}
For the proxy OOD task, we define the latent confounder $Z \in \{0, 1, 2\}$ and the proxy $S \in \{0, 1, 2\}$ with $X \in \mathbb{R}^3$ as the observed feature vector. The conditional distribution of $S$ given $Z$, denoted as $P(S \mid Z)$, is modeled with different probabilities depending on the value of $Z$. Specifically, the conditional probabilities for each value of $Z$ are given by:
\[
P(S = 0 \mid Z = 0) = 0.7, \quad P(S = 1 \mid Z = 0) = 0.3,
\]
\[
P(S = 0 \mid Z = 1) = 0.8, \quad P(S = 1 \mid Z = 1) = 0.1, \quad P(S = 2 \mid Z = 1) = 0.1,
\]
\[
P(S = 0 \mid Z = 2) = 0.1, \quad P(S = 2 \mid Z = 2) = 0.9.
\]

The feature vector $X = (X_1, X_2, X_3)$ is generated conditionally on $Z$ as follows:
\[
X_1 = -Z_x + Z + \epsilon_1, \quad X_2 = 2Z_x + \epsilon_2, \quad X_3 = 3Z_x + \epsilon_3,
\]
where $Z_x \sim \mathcal{N}(0, 1)$ is a Gaussian random variable, and $\epsilon_1 \sim \mathcal{N}(0, 0.2)$, $\epsilon_2 \sim \mathcal{N}(0, 1)$, $\epsilon_3 \sim \mathcal{N}(0, 1)$ are noise terms.

Finally, the binary target variable $Y$ is generated as a Bernoulli random variable based on a logistic function of $X$ and $Z$. The conditional probability of $Y$ given $X$ and $Z$ is:
\[
P(Y = 1 \mid X, Z) = \sigma(f(X, Z)),
\]
where $\sigma(\cdot)$ is the sigmoid function and the argument $\text{prob}(X, Z)$ is:
\[
f(X, Z) =
\begin{cases}
-0.6X_1 + 2.7X_2 + 0.6X_3 - 0.6Z + 1.5 & \text{if } Z = 0, \\
1.5X_1 + (1.2 - 6Z)X_2 + 2.1X_3 - 0.6Z + 1.5 & \text{if } Z \neq 0.
\end{cases}
\]

For the training set, the latent variable $Z$ is sampled from a multinomial distribution with probabilities $P(Z) = [0.15, 0.35, 0.5]$. For the test set, the probabilities are shifted to $P(Z) = [0.7, 0.2, 0.1]$ to simulate distribution shift.

\paragraph{Multi-source OOD Task}
For the multi-source OOD task, we also define $Z \in \{0, 1, 2\}$ as the latent confounder and $X \in \mathbb{R}^3$ as the feature vector. The conditional distribution $P(X \mid Z)$ follows the same structure as in the proxy OOD task.

The difference lies in how the training and test distributions for $Z$ are defined. During training, we use multiple source domains, each with different distributions of $Z$. For example, we sample from three source domains with distributions:
\[
P(Z) = [0, 0.7, 0.3], \quad P(Z) = [0.8, 0.1, 0.1], \quad P(Z) = [0.1, 0.0, 0.9].
\]
In the test domain, the distribution shifts to:
\[
P(Z) = [0.7, 0.2, 0.1].
\]

The target variable $Y$ is generated similarly to the proxy OOD task, with $P(Y \mid X, Z)$ following a logistic function based on $X$ and $Z$.

\subsection{Real Datasets}
\label{appendix:real_data}
We construct our datasets using the \texttt{folktables} Python package~\citep{ding2021retiring}. Below, we provide detailed information about each dataset and the rationale behind our splits.

\paragraph{ACSEmployment}
The ACSEmployment dataset is used to predict whether an individual is employed. Disability status is the latent confounder ($Z$), where $Z=1$ denotes individuals with disabilities, and $Z=0$ denotes individuals without disabilities. This setup models real-world conditions where disabled individuals might face different employment opportunities and challenges. Public insurance coverage and independent living status serve as proxies for disability status. In the training set, $\Ptr{Z=0}=0.95$ and $\Ptr{Z=1}=0.05$, while in the test set, $\Pte{Z=0}=0.05$ and $\Pte{Z=1}=0.95$. Both the training and test sets contain 170,000 samples each. The input has 54 features.

\paragraph{ACSPublicCoverage}
The ACSPublicCoverage dataset is used to predict whether an individual is covered by public health insurance. Age is the latent confounder ($Z$), where $Z=1$ denotes individuals with age above the mean, and $Z=0$ denotes individuals with age below the mean. This setup models real-world conditions where age could affect coverage of health insurance. We generate a binary synthetic proxy \(S\) for \(Z\). \(P(S=i \mid Z = i) = 0.8\) for \(i = 0,1\). In the training set, $\Ptr{Z=0}=0.95$ and $\Ptr{Z=1}=0.05$, while in the test set, $\Pte{Z=0}=0.05$ and $\Pte{Z=1}=0.95$. Both the training and test sets contain 4205 samples each. The input has 41 features.

\paragraph{ACSTravelTime}
The ACSTravelTime dataset is used to predict whether an individual has a commute to work that is longer than 20 minutes. State is the latent confounder ($Z$), where $Z=1$ denotes NY, and $Z=0$ denotes CA. This setup models real-world conditions where state could affect travel time. We generate a binary synthetic proxy \(S\) for \(Z\). \(P(S=i \mid Z = i) = 0.95\) for \(i = 0,1\). In the training set, $\Ptr{Z=0}=0.95$ and $\Ptr{Z=1}=0.05$, while in the test set, $\Pte{Z=0}=0.05$ and $\Pte{Z=1}=0.95$. Both the training and test sets contain 91200 samples each. The input has 41 features.

\paragraph{ACSIncome}
The ACSIncome dataset is used to predict an individual's income. The state of residence ($Z$) is the latent confounder. We created three synthetic splits for training:
\begin{itemize}
    \item 20\% Puerto Rico (PR) and 80\% California (CA)
    \item 20\% South Dakota (SD) and 80\% California (CA)
    \item 20\% PR, 20\% SD, and 60\% CA
\end{itemize}
Testing is performed on a split of 90\% PR, 5\% SD, and 5\% CA. The training set comprises 10,502 samples, while the test set contains 3,481 samples. The input has 76 features.

To create significant shifts in \( P(Y \mid X) \) across subpopulations, we first trained a classifier on data from California and tested it on data from all other states. Additionally, we trained a classifier for each state and tested it within the same state. We selected Puerto Rico and South Dakota as they exhibited the largest accuracy disparities compared to California. This method simulates real-world scenarios where economic conditions and job markets vary significantly across states, thereby affecting income levels.

\paragraph{ACSMobility}
The ACSMobility dataset is used to predict whether an individual had the same residential address one year ago. The state of residence ($Z$) is the latent confounder. The rest of the setup is the same as ACSIncome.

\subsection{Implementation Details}
\label{appendix:implementation_details}

The training was performed on an 11th Gen Intel(R) Core(TM) i7-11390H CPU and on an Intel(R) Xeon(R) Gold 6230 CPU @ 2.10GHz. The experiments took approximately 48 hours to run. Our architecture employs a two-hidden-layer neural network for \(\phi(\cdot)\). Furthermore, we use a Batch Normalization layer before the softmax layer to stabilize training.

We impose a constraint on the matrix \( M \) such that the sum of the values in each row equals 1, and we clip all values to remain within the range of 0 to 1. In Theorem~\ref{theorem:regularization}, we show how encouraging distributions with low variance leads to recovering the ground-truth latent distribution. The regularization parameter \(\lambda\) is tuned on the validation set from \{1e-3,1e-2,1e-1,1e0,1e1\}.

For optimization, we utilize the Adam optimizer with a learning rate of \(1 \times 10^{-3}\).  While training \(\phi(\cdot)\), we run the model for 100 epochs and perform early stopping based on a held-out validation set from the same distribution as the training set.

For the mixture of experts model, we train the model for 25 epochs with early stopping based on a validation set from the same distribution as the training set.

We do not assume \(n_z\) a priori. Instead, we determine \(n_z\) during training. To determine the optimal \(n_z\), we start with \(n_z = 1\) and train our encoder-decoder model to estimate the latent distribution \(P(X, S, Z)\). We incrementally increase \(n_z\) until a stopping criterion based on the validation reconstruction loss of \(S\) is met. Specifically, we stop increasing \(n_z\) at \(k\) if $\mathcal{L}_{S}(k+1) \geq \mathcal{L}_{S}(k)$, where $\mathcal{L}_{S}(k)$ denotes the reconstruction loss at \(k\) on the validation set.

The justification for this method is straightforward. Let $P(X, S, Z)$ represent the ground truth distribution, and let $n'_z$ denote the true number of categories of $Z$. Note that $P$ achieves the minimum reconstruction loss. If another latent distribution $Q(X, S, Z)$ with $n_z$ matches the observed distribution and achieves the minimum reconstruction loss (i.e., $P(X,S) = Q(X,S)$), then $\sum_z P(S|Z)P(Z|X) = \sum_z Q(S|Z)Q(Z|X)$. Since $P(S|Z)$ is full rank, this implies $n_z \geq n'_z$. Figure~\ref{fig:val_loss} illustrates the reconstruction loss versus different values of $n_z$ on ACSIncome. The plots shows that \(n_z = 3\) should be the ground-truth \(n_z\) and this agrees with the dataset.

\begin{figure}[htbp]
\centering
\begin{tikzpicture}
\begin{axis}[
    title={Validation Loss for different \(n_z\) for ACS Income Dataset},
    xlabel={\( n_z \)},
    ylabel={\( L_S^{\text{val}} \)},
    xmin=1, xmax=6,
    ymin=1.0, ymax=2.0,
    xtick={1,2,3,4,5,6},
    ytick={1.0,1.1,1.2,1.3,1.4,1.5,1.6,1.7,1.8,1.9},
]

\addplot[
    color=blue,
    mark=square,
    ]
    coordinates {
    (1,1.7922)(2,1.5322)(3,1.3275)(4,1.3306)(5,1.3413)(6,1.336)
    };
\end{axis}
\end{tikzpicture}
\caption{Plot of \( L_S^{\text{val}} \) over different \(n_z\).}
\label{fig:val_loss}
\end{figure}

\paragraph{Baselines}
We implement GroupDRO~\citep{DROwithlabel}, IRM~\citep{IRM}, V-REx~\citep{krueger2021out}, DeepCORAL~\citep{deepcoral}, and DANN~\citep{dann} using the DomainBed code~\citep{domainbed}. We use early stopping based on a validation set for all these baselines. We implement ProxyDA ~\citep{ProxyDA} using their publicly available GitHub repository. For V-REx, IRM, and DeepCORAL, we tune the regularization penalty from \{1e-2, 1e-1, 1e0, 1e1\}. For GroupDRO, we tune the group-step size from \{1e-2, 1e-1, 1e0, 1e1\}. For ProxyDA, we used random noise for the other variable. We subsampled all training datasets on the real tasks to 5000 samples respectively due to memory constraints. We tuned three hyperparameters: cme, m0, and scale using a validation set from the same distribution as the training set, within the range \{1e-4, 1e-3, 1e-2, 1e-1, 1e0\}. For the ACS Employment dataset, we transformed the data into a lower dimension using Gaussian Random Projection as suggested by \citet{ProxyDA}, tuning the number of dimensions from \{2, 4, 8, 16\}.

\subsection{Additional Results}
\label{appendix:additional_auc}
Table~\ref{tab:real_data_auc} provides the AUC for the experiments on real data in Section~\ref{sec:results}.
\begin{table}
    \centering
    \caption{AUC Performance on Real Data (Proxy and Multi-source OOD tasks)}
    \label{tab:real_data_auc}
    \begin{tabular}{lccccc}
        \toprule
        \multirow{2}{*}{Method} & Employment & Income & Mobility & Public Coverage & Travel Time \\
        \cmidrule(lr){2-2} \cmidrule(lr){3-3} \cmidrule(lr){4-4} \cmidrule(lr){5-5} \cmidrule(lr){6-6}
         & AUC OOD $\uparrow$ & AUC OOD $\uparrow$ & AUC OOD $\uparrow$ & AUC OOD $\uparrow$ & AUC OOD $\uparrow$ \\
        \midrule
        ERM & 0.626 (0.009) & 0.810 (0.005) & 0.691 (0.003) & 0.736 (0.007) & 0.534 (0.022) \\
        GroupDRO & 0.667 (0.006) & 0.783 (0.010) & 0.665 (0.011) & 0.682 (0.009) & 0.500 (0.000) \\
        IRM & 0.671 (0.005) & 0.817 (0.005) & 0.685 (0.006) & 0.730 (0.009) & 0.582 (0.000) \\
        DeepCORAL & 0.652 (0.007) & 0.825 (0.006) & 0.695 (0.004) & 0.735 (0.004) & 0.558 (0.063) \\
        VREx & 0.623 (0.017) & 0.809 (0.006) & 0.673 (0.010) & 0.699 (0.009) & 0.500 (0.000) \\
        DANN & 0.667 (0.003) & 0.816 (0.009) & 0.699 (0.003) & 0.725 (0.007) & 0.588 (0.008) \\
        ProxyDA & 0.609 (0.008) & 0.718 (0.113) & 0.694 (0.146) & NaN & 0.629 (0.025) \\
        Ours & \textbf{0.718 (0.003)} & \textbf{0.849 (0.002)} & \textbf{0.704 (0.004)} & \textbf{0.742 (0.013)} & \textbf{0.652 (0.000)} \\
        \bottomrule
    \end{tabular}
\end{table}

\end{document}